\documentclass[10pt,twocolumn,letterpaper]{article}

\usepackage{cvpr}              

\usepackage[dvipsnames]{xcolor}

\definecolor{cvprblue}{rgb}{0.21,0.49,0.74}
\usepackage[pagebackref,breaklinks,colorlinks,citecolor=cvprblue]{hyperref}

\title{Generalizing soft actor-critic algorithms to discrete action spaces}

\newcommand*\samethanks[1][\value{footnote}]{\footnotemark[#1]}
\author{Le Zhang\thanks{Beijing Electronic Science and Technology Institute}\\
{\tt\small lezhang.thu@gmail.com}
\and
Yong Gu\thanks{Jiangxi University of Finance and Economics}\\
{\tt\small guyonggu@icloud.com}
\and 
Xin Zhao\samethanks[1]\\
{\tt\small zhao-x15@tsinghua.org.cn}
\and 
Yanshuo Zhang\samethanks[1]\\
{\tt\small zhang\_yanshuo@163.com}
\and 
Shu Zhao\samethanks[1]\\
{\tt\small zhaoshu0104@163.com}
\and
Yifei Jin\thanks{Hangzhou Dianzi University}\\
{\tt\small yifei\_jin001@163.com}
\and
Xinxin Wu\samethanks[1]\\
{\tt\small wxxbk@outlook.com}
}

\usepackage{url}
\usepackage{amssymb}

\usepackage{amsthm}
\theoremstyle{plain}
\newtheorem{theorem}{Theorem}[section]

\newtheorem{lemma}[theorem]{Lemma}

\theoremstyle{definition}

\theoremstyle{remark}

\usepackage{amsmath,amsfonts,bm}

\def\eqref#1{equation~\ref{#1}}

\def\1{\bm{1}}

\def\va{{\bm{a}}}

\def\vs{{\bm{s}}}

\DeclareMathAlphabet{\mathsfit}{\encodingdefault}{\sfdefault}{m}{sl}
\SetMathAlphabet{\mathsfit}{bold}{\encodingdefault}{\sfdefault}{bx}{n}

\newcommand{\E}{\mathbb{E}}

\newcommand{\R}{\mathbb{R}}

\DeclareMathOperator*{\argmax}{arg\,max}
\DeclareMathOperator*{\argmin}{arg\,min}

\usepackage{mathtools}
\newcommand{\multimath}[1]{{\substack{#1}}}
\graphicspath{ {{./figures/}} }
\usepackage{algorithm}
\usepackage{listings}
\usepackage{hyperref}

\usepackage{listings}
\usepackage{xcolor}

\definecolor{codegreen}{rgb}{0,0.6,0}
\definecolor{codegray}{rgb}{0.5,0.5,0.5}
\definecolor{codepurple}{rgb}{0.58,0,0.82}
\definecolor{backcolour}{rgb}{0.95,0.95,0.92}

\lstdefinestyle{mystyle}{
    backgroundcolor=\color{backcolour},   
    commentstyle=\color{codegreen},
    keywordstyle=\color{magenta},
    numberstyle=\tiny\color{codegray},
    stringstyle=\color{codepurple},
    basicstyle=\ttfamily\footnotesize,
    breakatwhitespace=false,         
    breaklines=true,                 
    captionpos=b,                    
    keepspaces=true,                 
    numbers=left,                    
    numbersep=5pt,                  
    showspaces=false,                
    showstringspaces=false,
    showtabs=false,                  
    tabsize=2
}

\newcommand{\pycode}[1]{\lstinline[language=Python]{#1}}

\begin{document}
\maketitle

\begin{abstract}
ATARI is a suite of video games used by reinforcement learning (RL) researchers to test the effectiveness of the learning algorithm. Receiving only the raw pixels and the game score, the agent learns to develop sophisticated strategies, even to the comparable level of a professional human games tester.
Ideally, we also want an agent requiring very few interactions with the environment. Previous competitive model-free algorithms for the task use the valued-based Rainbow algorithm without any policy head. In this paper, we change it by proposing a practical discrete variant of the soft actor-critic (SAC) algorithm. The new variant enables off-policy learning using policy heads for discrete domains.
By incorporating it into the advanced Rainbow variant, i.e., the ``bigger, better, faster'' (BBF), the resulting SAC-BBF improves the previous state-of-the-art interquartile mean (IQM) from 1.045 to 1.088, and it achieves these results using only replay ratio (RR) 2. By using lower RR 2, the training time of SAC-BBF is strictly one-third of the time required for BBF to achieve an IQM of 1.045 using RR 8. As a value of IQM greater than one indicates super-human performance, SAC-BBF is also the only model-free algorithm with a super-human level using only RR 2. The code is publicly available on GitHub at \url{https://github.com/lezhang-thu/bigger-better-faster-SAC}.
\end{abstract}

\section{Introduction}

Back in 2015, DeepMind developed the deep Q-network (DQN) \cite{mnih2015human} to tackle the tasks in the challenging domain of classic ATARI 2600 games, which is a suite of video games with a wide range of diverse environments. The algorithm uses Q-learning, with critical techniques like \emph{experience replay} and target networks only periodically updated.
Receiving only the raw ATARI frames and the game score, DQN enables the agent to develop sophisticated human-level strategies. A series of extensions to DQN follows in these years. A partial list includes: double DQN utilizing the idea of double learning \cite{sutton2018reinforcement}, prioritized experience replay \cite{schaul2015prioritized}, dueling network \cite{wang2016dueling} splitting the Q-network into separated representations of state values and action advantages, distributional Q-learning \cite{bellemare2017distributional} which explicitly models the distribution over the returns, NoisyNet \cite{fortunato2017noisy} that substitutes the standard linear layers with noisy ones for efficient exploration, etc.
With $n$-step learning as in \cite{mnih2016asynchronous}, the Rainbow \cite{hessel2018rainbow} algorithm combines all the advances above, serving as a strong baseline for later algorithms. For distributional Q-learning, we note a series of works of the theme, e.g., quantile regression (QR-DQN) \cite{dabney2018distributional}, implicit quantile networks (IQN) \cite{dabney2018implicit}, fully parameterized quantile function (FQF) \cite{yang2019fully} etc.

All the algorithms above are value-based and operate off-policy, i.e., the agent can improve the existing policy by utilizing data whose distribution may not match the policy.
The off-policy characteristic enables repeated optimization using the same data sampled from the replay buffer.
By contrast, in on-policy algorithms such as asynchronous advantage actor-critic (A3C) \cite{mnih2016asynchronous} etc., ATARI frames are used only once in training and then discarded. In the real world, we want the agent to learn \emph{efficiently} by requiring only limited feedback from the environment. It is the task for sample-efficient RL. For the sample-efficient RL, a widely adopted benchmark is ATARI 100K, which limits the number of ATARI frames returned to 400K (frame-skipping in ATARI introduces the extra factor of 4), corresponding to approximately two hours of real-time play. In contrast to on-policy algorithms, off-policy algorithms like Rainbow fit this task well. 

We can broadly classify sample-efficient RL algorithms into model-based approaches and model-free ones.
The model-based algorithms hinge on learning a world model.
For the model-free ones, value-based Rainbow variants have consistently been the dominant choice within this category of algorithms. Now, we take a look at what Rainbow does.

Rainbow uses the Q-network with output dimension $|\mathcal{A}|$ to represent the policy, where $|\mathcal{A}|$ is the number of the discrete actions. Given $(s,a)$, the agent infers the Q-value $Q(s,a)$ from the Q-network. Thus, Rainbow involves no policy distribution like $\pi(\cdot|s)$ as in A3C.

A natural question is whether one can incorporate a separate policy head representing $\pi(\cdot|s)$ into Rainbow. We note, however, that 
algorithms with policy heads like A3C or proximal policy optimization (PPO) \cite{schulman2017proximal} are on-policy.
For these on-policy algorithms, reusing the same data to the same extent as in Rainbow is not well-justified, empirically often leading to destructive policy updates.
Thus, it appears there are inherent conflicts between leveraging Rainbow's off-policy property to enhance sample efficiency and incorporating a policy head representing $\pi(\cdot|s)$ into Rainbow to improve the learning even better.

Now, let us switch to the setting of continuous action spaces. Among the most widely employed algorithms in this context, the SAC \cite{haarnoja2018soft} algorithm successfully integrates policy learning and off-policy learning of Q-values. In this paper, we transfer this characteristic to discrete action spaces. We note, however, that it is not straightforward: 1.~SAC is designed for the maximum entropy RL rather than for the standard maximum reward RL as Rainbow. 2.~The transferring seems unpromising, as previous works exist, say SAC-Discrete, which falls \emph{far behind} across almost all the tested environments compared to data-efficient Rainbow (DER).

In this paper, we present a discrete variant of SAC for standard maximum reward RL and prove its convergence from scratch. Integration of it into Rainbow is straightforward, as now both fit in the standard maximum reward RL and work for discrete action spaces. We test it on the most advanced Rainbow variant for ATARI 100K, i.e., the BBF algorithm. While with a 3x reduction of training time, the resulting algorithm SAC-BBF improves the previous state-of-the-art IQM from 1.045 to 1.088. Also, SAC-BBF is the \emph{only} model-free algorithm with a super-human level using only RR 2. Further improvements of IQM using SAC-BBF are promising by using larger RRs, fostering the development of even more competitive agents.


\section{Related work}
\subsection{Competitive representatives in ATARI 100K}
Kaiser et al.~\cite{kaiser2019model} introduced the ATARI 100K benchmark and proposed simulated policy learning (SimPLe), which utilizes video prediction models to train a policy within a learned world model. Overtrained Rainbow (OTRainbow) \cite{kielakrecent} and DER \cite{van2019use} can be seen as improved hyperparameter configurations of Rainbow \cite{hessel2018rainbow}, tailored for ATARI 100K. Srinivas et al.~\cite{laskin2020curl} presented contrastive unsupervised representations for RL (CURL), which employs contrastive learning to enhance image representation quality. With simple image augmentations, data-regularized Q (DrQ) \cite{yarats2020image} demonstrates superior performance compared to preceding algorithms. Self-predictive representations (SPR) \cite{schwarzer2020data} trains the agent to predict its latent state representations multiple steps into the future, achieving a notable performance improvement over previous methods. Scaled-by-resetting SPR (SR-SPR) \cite{d2022sample} significantly improves sample efficiency by utilizing a replay ratio (RR) as large as 16, achievable by periodically fully or partially resetting the agent's parameters \cite{nikishin2022primacy}. 
EfficientZero \cite{ye2021mastering}, built upon MuZero \cite{schrittwieser2020mastering}, introduces the self-supervised consistency loss from SimSam \cite{chen2021exploring} and utilizes other tricks of the prediction of value prefix instead of rewards, and dynamically adjusting the step for computing the value targets.
With these, it is the first algorithm to achieve super-human performance on the ATARI 100K benchmark. Micheli et al.~\cite{micheli2022transformers} proposed IRIS, where the agent learns within a world model composed of a discrete autoencoder and an autoregressive Transformer \cite{vaswani2017attention}.

While EfficientZero achieves human-level efficiency, it hinges on Monte Carlo tree search (MCTS) and learning a world model. Super-human levels, therefore, seem elusive for model-free RL agents.
The breakthrough is the ``bigger, better, faster'' (BBF) agent,  proposed by Schwarzer et al.~\cite{schwarzer2023bigger}. The BBF algorithm is built upon SP-SPR and is the only model-free RL agent capable of achieving human-level performance ($\mathrm{IQM}\ge 1.0$).
Compared to EfficientZero, it achieves slightly better IQM but exhibits significantly reduced computational requirements, resulting in at least a 4x reduction in runtime.

We note, however, that all these representative model-free sample-efficient RL algorithms use Rainbow variants with no explicit policy head representing $\pi(\cdot |s)$ as the internal backbone.

\subsection{Previous results on discrete variants of SAC}
A few results exist on applying SAC to discrete action spaces, although these algorithms work in maximum entropy RL framework as SAC. On ATRARI 100K, SAC-Discrete is the first such attempt  \cite{christodoulou2019soft}. In the community, however, SAC-Discrete is observed not to work for the toy environment like Pong\footnote{\url{https://github.com/yining043/SAC-discrete}}.
Target entropy scheduled SAC (TES-SAC) \cite{xu2021target}, proposed by Xu et al., employs an annealing method for the target entropy parameter instead of a fixed entropy target.
Experimental results on TES-SAC are only available for ATARI 1M, i.e., 10x higher sample complexity. TES-SAC, however, is generally inferior to DER, which learns only in ATARI 100K.

Zhou et al.~revisited the concept of discrete SAC \cite{zhou2022revisiting}. The proposed variant utilizes entropy-penalty and double average Q-learning with Q-clip.
Zhou et al.~only reported the agent performance in ATRAI 1M and 10M.
Considering the higher sample-complexity setting of ATRAR 1M instead of 100K, when compared to Rainbow on ATRAR 1M (presented in Table 3 of \cite{kaiser2019model}), the proposed algorithm exhibits suboptimal performance across most tested environments.

Upon reflection, these preceding attempts neglect a simple yet crucial technique for variance reduction. In SAC-BBF, we use it, and experimentally, it is the single most important trick for successfully adapting SAC to discrete domains.

\subsection{Previous algorithms combining Q-learning with actor-critic}
Researchers previously proposed actor-critic algorithms with experience replay buffers. These algorithms also fit the category of combining policy heads with Q-learning. We only review two representatives, i.e., ACER by Wang et al.~\cite{wang2016sample} and Reactor by Gruslys et al.~\cite{gruslys2017reactor}. These two algorithms are most closely based upon A3C \cite{gruslys2017reactor}.

ACER introduces importance sampling truncation with bias correction.
Reactor introduces the $\beta$-LOO (i.e., leave-one-out) policy gradient algorithm. These innovations enable the reuse of the data in the replay buffer for policy updates. The update forms of these algorithms, however, 
do not have the theoretical elegance of SAC, involving importance weights, like $\min(c,\pi(\hat{a})/\mu(\hat{a}))(R(\hat{a})-V)\nabla\log\pi(\hat{a})$ for Reactor where $\mu$ is the behavior policy and $\hat{a}\sim \mu$, and $\bar{\rho}_t\nabla_\theta \log\pi_\theta(a_t|x_t)[Q^\mathrm{ret}(x_t,a_t)-V_{\theta_v}(x_t)]$ for ACER, where $\bar{\rho}_t=\min\{c,\rho_t\}$ with $\rho_t=\pi(a_t|x_t)/\mu(a_t|x_t)$.

Experimentally, the Reactor generally exceeds the performance of ACER. Results of the Reactor are reported on 500M training frames and are comparable to Rainbow \cite{gruslys2017reactor}.

\section{Preliminaries}
Consider a Markov decision process (MDP), defined as a tuple $(\mathcal{S}, \mathcal{A}, p, r, \rho_0, \gamma)$, where $\mathcal{S}$ and $\mathcal{A}$ represent the sets of possible states and actions, respectively. The transition function $p:\mathcal{S}\times \mathcal{A}\times \mathcal{S}\to \mathbb{R}$ represents $\mathrm{Pr}(\vs_{t+1}=s'|\vs_t=s,\va_t=a)$. The reward function $r:\mathcal{S}\times \mathcal{A}\to [r_{\min}, r_{\max}]$ is the expected value of the scalar reward when action $a$ is taken in state $s$. The initial state distribution is denoted by $\rho_0$, and $\gamma\in (0,1)$ is a discount factor. We use $\rho_\pi(\vs_t,\va_t)$ to denote the state-action marginals of
the trajectory distribution induced by a policy $\pi(\va_t|\vs_t)$.
The objective for optimization is defined as follows:
\begin{equation}
\label{eq:obj}
    J(\pi)=\sum_{t=0}^\infty \mathop{\E}_{(\vs_t,\va_t)\sim \rho_\pi}\left[\sum_{l=t}^\infty \gamma^{l-t}\mathop{\E}_\multimath{\vs_l\sim p\\\va_l\sim \pi}[r(\vs_l,\va_l)\mid\vs_t,\va_t]\right]\text{,}
\end{equation}
which aims to maximize the discounted expected reward for future states, given every state-action tuple $(\vs_t,\va_t)$, weighted by its probability $\rho_\pi$ under the current policy. 

\subsection{The SAC algorithm}
SAC works for maximum entropy RL, so the objective for SAC is as follows:
\begin{equation}
\begin{split}
    J(\pi)=\sum_{t=0}^\infty \mathop{\E}_{(\vs_t,\va_t)\sim \rho_\pi}\biggl[&\sum_{l=t}^\infty \gamma^{l-t}\mathop{\E}_\multimath{\vs_l\sim p\\\va_l\sim \pi}[r(\vs_l,\va_l)\\
    &+\alpha \mathcal{H}(\pi(\cdot|\vs_t))
    \mid\vs_t,\va_t]\biggr]\text{,}
\end{split}
\end{equation}
where $\alpha$ determines the relative importance of the entropy term.

The soft policy iteration of SAC alternates between soft policy evaluation and soft policy improvement.

\noindent{\textbf{Soft policy evaluation:}} For a fixed policy, the soft Q-value can be computed by iteratively applying the following Bellman
backup operator
\begin{equation}
    \mathcal{T}^\pi Q(\vs_t,\va_t)\triangleq r(\vs_t, \va_t)+\gamma\mathop{\E}_{\vs_{t+1}\sim p}[V(\vs_{t+1}]]\text{,}
\end{equation}
where $V(\vs_t)=\mathop{\E}_{\va_t\sim\pi}[Q(\vs_t,\va_t)-\alpha\log\pi(\va_t|\vs_t)]$.

\noindent{\textbf{Soft policy improvement:}} SAC updates the policy according to 
\begin{equation}
\label{eq:d_kl}
\pi_\mathrm{new}=\argmin_{\pi'\in\Pi}D_\mathrm{KL}\biggl(\pi'(\cdot|\vs_t)\|\frac{\exp(Q^{\pi_\mathrm{old}(\vs_t,\cdot)}/\alpha)}{Z^{\pi_\mathrm{old}}(\vs_t)}\biggr)\text{,}
\end{equation}
where $\Pi$ is a parameterized family of
distributions, $D_\mathrm{KL}$ is Kullback-Leibler divergence, and $Z^{\pi_\mathrm{old}}(\vs_t)$ plays the role of a normalizing constant. SAC updates the policy towards the exponential of the Q-function, and $D_{KL}(\cdot\|\cdot)$ serves as the projection so that the constraint $\pi\in\Pi$ is satisfied.

Let the parameters of the Q-network and the policy network be $\phi$ and  $\theta$ resp. SAC works for continuous domains, so the Q-network $Q_\phi(\vs_t,\va_t)$ is of two inputs $\vs_t$ and $\va_t$. Thanks to it, $Q_\phi(\vs_t,\va_t)$ in SAC is thus differentiable w.r.t.~$\va_t$. To utilize it for a lower variance estimator, SAC applies the re-parameterization \cite{kingma2013auto,pmlr-v89-xu19a} trick of $\va_t$ as follows:
\begin{equation}
    \va_t=f_\theta(\epsilon_t;\vs_t)\text{,}
\end{equation}
where $\epsilon_t$ is independent noise, which follows the distribution, say, a spherical Gaussian. With this, the approximate gradient for the optimization in Eq.~\ref{eq:d_kl} is thus 
\begin{equation}
\begin{split}
    \nabla_\theta\alpha \log\pi_\theta(\va_t|\vs_t)+\bigl(&\nabla_{\va_t}\alpha\log\pi_\theta(\va_t|\vs_t)\\
    &-\nabla_{\va_t}Q(\vs_t,\va_t)\bigr)\nabla_\theta f_\theta(\epsilon_t;\vs_t)\text{.}
\end{split}
\end{equation}

\subsection{The BBF algorithm}
BBF uses an advanced version of the SR-SPR agent. So, we first review the techniques of SP-SPR. The architecture of SP-SPR is similar to the one depicted in Fig.~\ref{fig:architecture}, although with dashed boxes removed.
The learning process of SP-SPR integrates Q-learning from the Rainbow algorithm with self-predictive representation learning, which predicts the latent state representations multiple steps ahead.
SP-SPR employs the cosine similarity loss, as illustrated in Fig.~\ref{fig:architecture}, to compute the prediction loss.

An important hyper-parameter in SP-SPR is the \emph{replay ratio}, which denotes the ratio of learning updates relative to the environmental steps.
SP-SPR scales the replay ratio up to 16.
Larger replay ratios may impede the learning process \cite{d2022sample}. SP-SPR manages to leverage replay ratio scaling capabilities by periodically resetting parts of parameters \cite{nikishin2022primacy}.

Compared to SR-SPR, the BBF algorithm scales the network capacity. 
As SR-SPR collapses as network size increases \cite{schwarzer2023bigger}, BBF needs more tricks.
It uses the AdamW \cite{loshchilov2017decoupled} optimizer with weight decay. 
Besides, $n$ in $n$-step learning and $\gamma$ are dynamically decreased and increased resp., following exponential schedules.
Also, BBF removes NoisyNets, which helps little in improving the performance. Lastly, BBF uses harder resets of the convolutional layers for possible more regularization. 


\subsection{Evaluation metrics}
Most prior algorithms, including SPR, use mean and median scores across tasks for evaluation, which neglect the statistical uncertainty as experiments typically consist of only a small number of training runs. 
For the sake of reliable evaluation in the few-run deep RL regime, Agarwal et al.~presented a NeurIPS'21 outstanding paper \cite{agarwal2021deep}, accompanied by an open-source library \pycode{rliable}. In particular, for aggregate metrics, Agarwal et al.~\cite{agarwal2021deep} advocated the use of the interquartile mean (IQM), which is the average score of the middle 50\% runs combined across all games and seeds. IQM is ``easily applied with 3-10 runs per task.'' \cite{agarwal2021deep} Other evaluation metrics include the mean and median normalized scores, as well as the optimality gap, which quantifies the amount by which the algorithm fails to obtain human-level performance.

\section{A discrete variant of SAC for standard maximum reward RL}
In this section, we extend SAC beyond the maximum entropy RL. In standard maximum reward RL, we set $\alpha$ as zero. In SAC, however, the presence of the term $\exp(Q^{\pi_\mathrm{old}}(\mathbf{s}_t,\cdot)/\alpha)$ in Eq.~\ref{eq:d_kl} (the objective of soft policy improvement step of SAC) indicates setting $\alpha$ as zero makes no sense. So, it raises the concern whether the theory of SAC would break. 
To address this concern, we re-prove all the lemmas and the theorem. We rigorously follow the original proofs so one can easily verify the correctness.

\subsection{Policy evaluation}
Similar to Sarsa \cite{sutton2018reinforcement}, the Q-value can be computed iteratively for a fixed policy. We can start with any function $Q:\mathcal{S}\times \mathcal{A}\rightarrow \R$ and repeatedly apply the Bellman backup operator $\mathcal{T}^\pi$, defined as:
\begin{equation}
\label{eq:bellman_backup}
    \mathcal{T}^\pi Q(\vs_t,\va_t)\triangleq r(\vs_t, \va_t)+\gamma\mathop{\E}_{\multimath{\vs_{t+1}\sim p\\\va_{t+1}\sim \pi}}[Q(\vs_{t+1},\va_{t+1})]\text{.}
\end{equation}

\begin{lemma}[Policy Evaluation]
\label{lem:eval}
Let $\mathcal{T}^\pi$ be the Bellman backup operator defined in Eq.~\ref{eq:bellman_backup}, and let $Q^0:\mathcal{S}\times \mathcal{A}\rightarrow \R$ be a mapping. We define $Q^{k+1}=\mathcal{T}^\pi Q^k$. Then, as $k$ approaches infinity, the sequence $Q^k$ converges to the Q-value of $\pi$.
\end{lemma}

\begin{proof}
The proof follows by applying standard convergence results for policy evaluation \cite{sutton2018reinforcement}.
\end{proof}

\subsection{Policy improvement}
In policy improvement, we depart from SAC's method of updating the policy towards the exponential of the Q-function. Besides, we eliminate the projection needed for satisfying the constraint $\pi\in\Pi$.
Concretely, for each state, our policy is updated based on the following equation:
\begin{equation}
\label{eq:max_policy}
    \pi_\mathrm{new}(\cdot|\vs_t)=\argmax_{\pi'\in \Pi}\mathop{\E}_{\va_t\sim \pi'}[Q^{\pi_\mathrm{old}}(\vs_t,\va_t)]\text{.}
\end{equation}

We now demonstrate that the policy update described in Eq.~\ref{eq:max_policy} leads to an improved policy w.r.t.~the objective stated in Eq.~\ref{eq:obj}.
\begin{lemma}[Policy Improvement]
\label{lem:improve}
   Let $\pi_\mathrm{old}\in \Pi$, and let $\pi_\mathrm{new}$ be the optimizer of the maximization problem defined in Eq.~\ref{eq:max_policy}. Then, for all $(\vs_t,\va_t)\in \mathcal{S}\times \mathcal{A}$, it holds that $Q^{\pi_\mathrm{new}}(\vs_t, \va_t)\ge Q^{\pi_\mathrm{old}}(\vs_t, \va_t)$.
\end{lemma}

\begin{proof}
    See Appendix \ref{sec:proof_lemma_improve} in the supplementary material.
\end{proof}

With the two lemmas mentioned above, we can state the following theorem on convergence to the optimal policy among the policies in $\Pi$.

\begin{theorem}[Policy Iteration]
\label{thm:policy_iteration}
    The repeated application of policy evaluation and policy improvement from any $\pi\in\Pi$ converges to a policy $\pi^*$ such that $Q^{\pi^*}(\vs_t,\va_t)\ge Q^\pi(\vs_t,\va_t)$ for all $\pi\in\Pi$ and $(\vs_t,\va_t)\in\mathcal{S}\times \mathcal{A}$.
\end{theorem}

\begin{proof}
    See Appendix \ref{sec:policy_iteration} in the supplementary material.
\end{proof}

\subsection{A practial algorithm}
As a practical algorithm, we employ function approximators, such as deep neural networks, to represent both the Q-function and the policy.

\subsubsection{Variance reduction}
\label{sec:main_variance_reduction}
Consider the parameterized Q-function $Q_{\phi}(\vs_t)$ (with parameters $\phi$) and the parameterized policy $\pi_\theta$ (with parameters $\theta$). 
We have the following lemma on the optimization in Eq.~\ref{eq:max_policy}.

\begin{lemma}
\label{lem:objective_grad}
    The objective in Eq.~\ref{eq:max_policy} can be optimized with stochastic gradients: 
    \begin{equation}
    \label{eq:grad_Q}
    \begin{split}
        \nabla_\theta \mathop{\E}_{\va_t\sim \pi_\theta}&[Q_{\phi_\mathrm{old}}(\vs_t,\va_t)]=\\
        &\mathop{\E}_{\va_t\sim \pi_\theta}\left[Q_{\phi_\mathrm{old}}(\vs_t,\va_t)\nabla_\theta\log\pi_\theta(\va_t|\vs_t)\right]\text{.}
    \end{split}    
    \end{equation}
\end{lemma}

\begin{proof}
    See Appendix \ref{sec:objective_grad} in the supplementary material.
\end{proof}

In SAC, the Q-network is a neural network with two inputs $\vs_t$ and $\va_t$. In contrast, for discrete action spaces, the Q-network only receives one input $\vs_t$, outputting a vector of the dimension $|\mathcal{A}|$. For this case, we have the following lemma for the gradient estimator with reduced variance:
\begin{lemma}[Variance reduction]
\label{lem:variance_reduction}
    The following two gradient estimators are equal:
\begin{equation}
\label{eq:variance_code}
\begin{split}
        &\mathop{\E}_{\va_t\sim \pi_\theta}\biggl[\Bigl(Q_{\phi_\mathrm{old}}(\vs_t,\va_t)-\sum_{\va'\in\mathcal{A}}\pi_{\theta_\mathrm{old}}(\va'|\vs_t)Q_{\phi_\mathrm{old}}(\vs_t,\va')\Bigr)\\   &\quad\quad\quad\quad\nabla_\theta\log\pi_\theta(\va_t|\vs_t)\biggr]\\
    =&\mathop{\E}_{\va_t\sim \pi_\theta}[Q_{\phi_\mathrm{old}}(\vs_t,\va_t)\nabla_\theta\log\pi_\theta(\va_t|\vs_t)]\text{.}
\end{split}
\end{equation}    
\end{lemma}

\begin{proof}
    See Appendix \ref{sec:variance_reduction} in the supplementary material.
\end{proof}

\subsubsection{An entropy bonus}
\label{sec:ent}
We augment the objective by adding an entropy bonus to the policy $\pi_\theta$ to discourage premature convergence to suboptimal deterministic policies.
Additionally, we replace the expectations in Eq.~\ref{eq:variance_code} with sample averages. The final gradient estimator for policy parameters $\theta$ is given by:
\begin{equation}
\label{eq:actor_loss}
    Q_{\phi_\mathrm{old}}(\vs_t,\va_t)\nabla_\theta\log\pi_\theta(\va_t|\vs_t)+\beta \nabla_\theta\mathcal{H}(\pi(\cdot |\vs_t))\text{,}
\end{equation}
where the hyperparameter $\beta$ controls the strength of exploration encouragement. We note $\va_t\sim\mathcal{\pi_\theta}$, so whenever $\mathbf{s}_t$ is used, it samples a new action $\mathbf{a}'_t\sim \pi_\theta$. It contrasts ACER or Reactor, where off-policy learning for policy updates always centers over the action in the replay buffer for $\vs_t$.

We linearly anneal $\beta$ from an initial constant value to 0. We then keep $\beta=0$ till the training ends. The linear annealing scheme plays a role in better performance, for which we defer the details to the experiment section.


\begin{figure*}
  \centering
  \includegraphics[scale=.42]{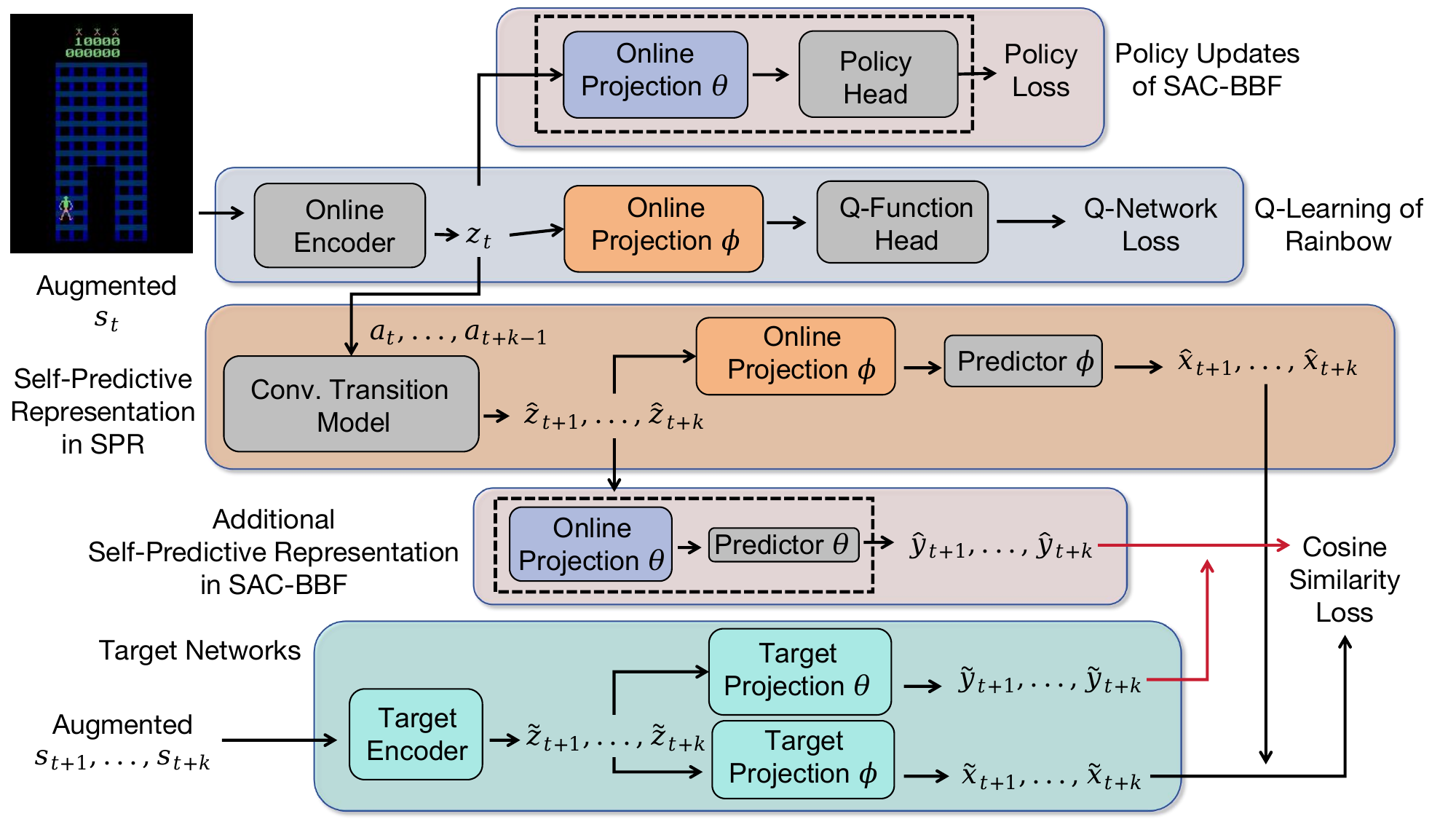}
  \caption{\textbf{Architecture of SAC-BBF.} Modules within dashed boxes represent additions introduced by SAC-BBF. In this architecture, the target modules typically correspond to exponentially moving average (EMA) versions of the online counterparts. The encoders used are Impala-CNN \cite{espeholt2018impala}, with each layer's width increased by a factor of four. Regarding the input of actions into the ``conv.~transition model,'' each action is encoded as a one-hot vector and then broadcasted to every location of the convolutional output from the encoder. The remaining modules in the architecture consist of linear layers.}
  \label{fig:architecture}
\end{figure*}
\section{Integrating SAC with BBF}
In this section, we integrate the SAC variant discussed in the previous section with BBF. 
We depict the new architecture in Fig.~\ref{fig:architecture}. Notably, all networks from BBF remain unaltered. SAC-BBF introduces three additional modules (excluding the target module counterparts): ``online projection $\theta$,'' ``policy head,'' and ``predictor $\theta$.'' We implement these modules as simple linear layers.

\subsection{Modifying target values for training the Q-network}
In addition to the network modifications, SAC-BBF also alters the target used in the $n$-step learning (Q-learning) of BBF. The target for BBF's $n$-step learning is defined as follows:
\begin{equation}
\label{eq:targ_sac}
    \biggl(\sum^{n-1}_{k=0}\gamma^k r_{t+k+1}\biggr)+\gamma^n Q_{\phi_\mathrm{targ}}\bigl(\vs_{t+n},\argmax_{\va'}Q_\phi(\vs_{t+n},\va')\bigr)\text{,}
\end{equation}
where $r_{t+k+1}$ represents the reward obtained from the state-action pair $(\vs_{t+k+1},\va_{t+k+1})$, and $\phi_\mathrm{targ}$ denotes the corresponding set of parameters for the target Q-network. 

In SAC-BBF, we substitute $\argmax_{\va'}Q_\phi(\vs_{t+n},\va')$ in Eq.~\ref{eq:targ_sac} with $\va'\sim \pi_\theta(\cdot|\vs_{t+n})$. This replacement aligns with the Bellman backup operator $\mathcal{T}^\pi$ defined in Eq.~\ref{eq:bellman_backup}.

\subsection{Incorporating additional terms in the prediction loss}

SAC-BBF introduces a projection layer for the policy $\pi_\theta$, which prompts an extension of the self-predictive representations (SPR) loss to this layer. The modified SPR loss in SAC-BBF is defined as follows:
\begin{equation}
    -\frac{1}{2k}\sum_{\multimath{0<j\le k\\v\in\{x,y\}}}\Bigl(\frac{\hat{v}_{t+j}}{\|\hat{v}_{t+j}\|_2}\Bigr)^T\Bigl(\frac{\tilde{v}_{t+j}}{\|\tilde{v}_{t+j}\|_2}\Bigr)\text{,}
\end{equation} 
where $\hat{x}_{t+j},\tilde{x}_{t+j},\hat{y}_{t+j},\tilde{y}_{t+j}$ correspond to the vectors depicted in Fig.~\ref{fig:architecture}.

\subsection{Implementing a new policy loss}
In SAC-BBF, we use Eq.~\ref{eq:actor_loss} for updating the policy head. We note the learning process in SAC-BBF differs from soft Q-learning \cite{haarnoja2017reinforcement}, where the policy network only acts as an approximate sampler from the soft Q-function.

\section{Experiments}
We build the implementation of SAC-BBF over that of BBF. To ensure a fair comparison, we maintain consistency with BBF in all hyperparameters and training configurations whenever applicable. We set $F=\text{40K}$ and the initial value of $\beta$ as 0.01.

We carry out a series of experiments focusing on the following aspects:
\begin{enumerate}
    \item Investigating the role of variance reduction, as discussed in Sec.~\ref{sec:main_variance_reduction}, in the effective functioning of an agent.
    \item Assessing the effectiveness of annealing $\beta$ and examining the impact of employing a sampling strategy during evaluation.
    \item Presenting the results of SAC-BBF, highlighting its ability to achieve new benchmark IQM results.
    \item Exploring miscellaneous factors such as training and inference times compared to BBF.
\end{enumerate}

For all variants of SAC-BBF, we obtain the results through 10 independent runs and evaluate them over 100 episodes upon completion of training.

\begin{algorithm}[t]
\caption{The code for randomly sampling the five environments}
\label{alg:code}
\definecolor{codeblue}{rgb}{0.25,0.5,0.5}
\definecolor{codekw}{rgb}{0.85, 0.18, 0.50}
\lstset{
  backgroundcolor=\color{white},
  basicstyle=\fontsize{7.5pt}{7.5pt}\ttfamily\selectfont,
  columns=fullflexible,
  breaklines=true,
  captionpos=b,
  commentstyle=\fontsize{7.5pt}{7.5pt}\color{codeblue},
  keywordstyle=\fontsize{7.5pt}{7.5pt}\color{codekw},
}
\begin{lstlisting}[language=python]
import numpy as np
    
def randomly_5(games, seed):
    np.random.seed(seed)
    games = np.asarray(games)
    np.random.shuffle(games)
    return games[:5]
\end{lstlisting}
\end{algorithm}
\subsection{Selecting subsets of environments for ablation studies}
For the ablation studies, we prioritize carbon reduction. We thus restrict the experiments to various randomly selected subsets of 5 games from the complete suite of 26 games in the Atari 100K benchmark. To ensure randomness in selecting these subsets, you can use the Python function provided in Algorithm \ref{alg:code}. We state the seeds for each experiment at the beginning of the following subsections.

\begin{table}
\caption{\textbf{Comparison of SAC-BBF w.o.~variance reduction and SAC-BBF.} Average scores for these games are listed. The human-normalized IQM and other statistics over the five randomly selected environments. The values of these statistics, therefore, differ from those calculated over the full suite of 26 environments in Atari 100K.}
\label{tab:variance_reduction}
\begin{center}
\begin{small}
\resizebox{.8\columnwidth}{!}{%
\begin{tabular}{rrr}
\hline
\hline
Game &
\begin{tabular}{r}
     SAC-BBF w.o. \\
     variance reduction 
\end{tabular}
& SAC-BBF \\
\hline
KungFuMaster & 886.6 & \textbf{17746.9} \\
Krull & 0.07 & \textbf{7884.82} \\
Frostbite & 58.77 & \textbf{2169.26} \\
RoadRunner & 729.6 & \textbf{24165.6} \\
Jamesbond & 60.9 & \textbf{1202.7} \\
\hline
IQM ($\uparrow$) & -0.008 & \textbf{2.493} \\
Optimality Gap ($\downarrow$) & 1.252 & \textbf{0.167} \\
Median ($\uparrow$) & 0.027 & \textbf{3.083} \\
Mean ($\uparrow$) & -0.252 & \textbf{2.906} \\
\hline
\end{tabular}
}
\end{small}
\end{center}
\vskip -0.1in
\end{table}
\begin{table}
\caption{\textbf{Comparing variants with constant $\beta$ during training and random action selection during evaluation.}}
\label{tab:anneal_random}
\begin{center}
\begin{small}
\resizebox{\columnwidth}{!}{%
\begin{tabular}{rrrrr}
\hline
\hline
Game & Human &
\begin{tabular}{r}
     $\beta=0.01$ \\
     Greedy in eval.
\end{tabular} &
\begin{tabular}{r}
     $\beta=0.01$ \\
     Sampling in eval.
\end{tabular} &
SAC-BBF \\
\hline
Seaquest & 42054.7 & \textbf{1192} & 1154.06 & 1044.3 \\
Alien & 7127.7 & \textbf{1178.14} & 1017.58 & 1158.44 \\
CrazyClimber & 35829.4 & 57980 & 76177.5 & \textbf{84932.6} \\
Pong & 14.6 & 13.756 & 13.555 & \textbf{15.549} \\
Kangaroo & 3035.0  & 2870 & 2481.2 & \textbf{5288.6} \\
\hline
IQM ($\uparrow$)  & 1.000 & 0.470 & 0.583 & \textbf{0.750} \\
Optimality Gap ($\downarrow$) & 0.000 & 0.509 & 0.461 & \textbf{0.425} \\
Median ($\uparrow$) & 1.000 & 0.944 & 0.814 & \textbf{1.026} \\
Mean ($\uparrow$) & 1.000 & 0.793 & 0.907 & \textbf{1.180} \\
\hline
\end{tabular}
}
\end{small}
\end{center}
\vskip -0.1in
\end{table}

\begin{table*}
\caption{\textbf{Scores and aggregate metrics for BBF and competitive agents across the 26 Atari 100K games.} The scores are averaged over 5 seeds for SAC-Discrete, 30 seeds per game for SR-SPR, 3 for EfficientZero, 50 for BBF with RR 8, 14 for BBF with $=2$, and 10 for SAC-BBF with RR 2. Scores of SAC-Discrete are from \cite{christodoulou2019soft}. Statistics on SR-SPR, EfficientZero, and BBF with RR 8 are from \cite{schwarzer2023bigger}. The statistics within parentheses for BBF with RR 8, as well as the results for BBF with RR 2, are calculated based on the publicly available scores provided in the official repository of \cite{schwarzer2023bigger}. For the publicly available scores for BBF with RR 8, we note the numbers of independent runs vary in values in 52, 56, or 60. For consistency with the 50 seeds in the original paper of BBF, for each game, we randomly permute all its scores and take the first 50 runs. IQM of SR-SPR is fixed to 0.632 as in \cite{d2022sample}, rather than 0.631 in \cite{schwarzer2023bigger}, which is likely a typo.}
\label{tab:Atari_100K}
\begin{center}
\resizebox{.95\textwidth}{!}{%
\begin{tabular}{r|rrrrrrrr}
\hline
\hline
Game & Random & Human & SAC-Discrete &
SR-SPR
&
EfficientZero
&
\begin{tabular}{r}
    BBF \\
    RR2 
\end{tabular}
&
\begin{tabular}{r}
    BBF \\
    RR8 
\end{tabular}
&
\begin{tabular}{r}
    SAC-BBF \\
    RR2 
\end{tabular} \\
\hline
Alien & 227.8 & 7127.7 & 216.9 & 1107.8 & 808.5 & 1121.714	& \textbf{1173.2} & 1158.44 \\
Amidar & 5.8 & 1719.5 & 7.9 & 203.4 & 148.6 & 236.609	& \textbf{244.6} & 211.698 \\
Assault & 222.4 & 742.0 & 350.0 & 1088.9 & 1263.1 & 2004.509	& \textbf{2098.5} & 1846.01 \\
Asterix & 210.0 & 8503.3 & 272.0 & 903.1 & \textbf{25557.8} & 3169.785	& 3946.1 & 5641.45 \\
BankHeist & 14.2 & 753.1 & - & 531.7 & 351.0 & 768.835	& 732.9 & \textbf{866.61} \\
BattleZone & 2360.0 & 37187.5 & 4386.7 & 17671.0 & 13871.2 & 23681.428 & \textbf{24459.8} & 21961\\
Boxing & 0.1 & 12.1 & - & 45.8 &  52.7 & 77.362 & \textbf{85.8} & 84.097 \\
Breakout & 1.7 & 30.5 & 0.7 & 25.5 & \textbf{414.1} & 331.07 & 370.6& 327.044 \\
ChopperCommand & 811.0 & 7387.8 & - & 2362.1 & 1117.3 & 4251.571 & 7549.3 & \textbf{8825.6} \\
CrazyClimber & 10780.5 & 35829.4 & 3668.7 & 45544.1 & 83940.2 & 60864.5	& 58431.8 & \textbf{84932.6} \\
DemonAttack & 152.1 & 1971.0 & - & 2814.4 & 13003.9 & 18298.36 & 13341.4 & \textbf{19436.53} \\
Freeway & 0.0 & 29.6 & 4.4 & 25.4 & 21.8 & 23.125	& \textbf{25.5} & 16.456 \\
Frostbite & 65.2 & 4334.7 & 59.4 & \textbf{2584.8} & 296.3 & 2023.078 & 2384.8 & 2169.26 \\
Gopher & 257.6 & 2412.5 & - & 712.4 & \textbf{3260.3} & 1209.414 & 1331.2 & 1203.6 \\
Hero & 1027.0 & 30826.4 & - & 8524.0 & \textbf{9315.9} & 5741.821 & 7818.6 & 6958.27 \\
Jamesbond & 29.0 & 302.8 & 68.3 & 389.1 & 517.0 & 1124.642	& 1129.6 & \textbf{1202.7} \\
Kangaroo & 52.0 & 3035.0 & 29.3 & 3631.7 & 724.1 & 5032.071	& \textbf{6614.7} & 5288.6 \\
Krull & 1598.0 & 2665.5 & - & 5911.8 & 5663.3 & 8069.842	& \textbf{8223.4} & 7884.82 \\
KungFuMaster & 258.5 & 22736.3 & - & 18649.4 & \textbf{30944.8} & 16616.857 & 18991.7	& 17746.9 \\
MsPacman & 307.3 & 6951.6 & 690.9 & 1574.1 & 1281.2 & \textbf{2217.842} & 2008.3 & 1922.41 \\
Pong & -20.7 & 14.6 & -20.98 & 2.9 & \textbf{20.1} & 13.698 & 16.7 & 15.549 \\
PrivateEye & 24.9 & 69571.3 & - & \textbf{97.9} & 96.7 & 39.071 & 40.5 & 59.582 \\
Qbert & 163.9 & 13455.0 & 280.5 & 4044.1 & \textbf{14448.5} & 3245.339 & 4447.1 & 4234 \\
RoadRunner & 11.5 & 7845.0 & 305.3 & 13463.4 & 17751.3 & 26419 & \textbf{33426.8} & 24165.6 \\
Seaquest & 68.4 & 42054.7 & 211.6 & 819.0 & 1100.2 & 988.628	& \textbf{1232.5} & 1044.3 \\
UpNDown & 533.4 & 11693.2 & 250.7 & \textbf{112450.3} & 17264.2 & 15122.685 & 12101.7 & 34848.44 \\
\hline
Games $>$ Human & 0 & 0 & - & 9 & \textbf{14} & 11 & 12 & 13 \\
IQM ($\uparrow$) & 0.000 & 1.000 & - & 0.632 & 1.020 & 0.94	& 1.045 (1.035) & \textbf{1.088} \\
Optimality Gap ($\downarrow$) & 1.000 & 0.000 & - & 0.433 & 0.371 & 0.376 & 0.344 (\textbf{0.341}) & 0.359 \\
Median ($\uparrow$) & 0.000 & 1.000 & - & 0.685 & \textbf{1.116} & 0.754 & 0.917 (0.883) & 0.902 \\
Mean ($\uparrow$) & 0.000 & 1.000 & - & 1.272 & 1.945 & 2.175 & 2.247	(2.247) & \textbf{2.345} \\
\hline
\end{tabular}%
}
\end{center}
\vskip -0.1in
\end{table*}
\begin{figure*}
  \centering
  \includegraphics[scale=.46]{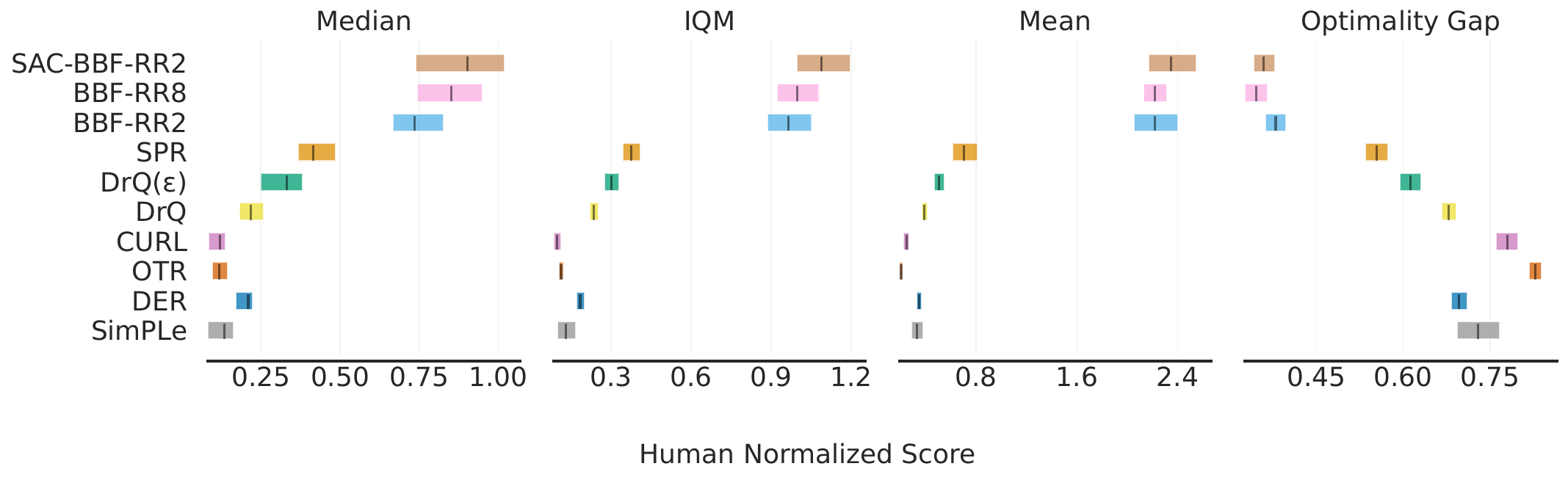}
  \caption{\textbf{Aggregate metrics with 95\% stratified bootstrap CIs for representatives of RL algorithms within the Atari 100K benchmark:} The results from SimPLe to SPR represent the default metrics provided in \pycode{rliable}. The data for BBF-RR2 and BBF-RR8 are from the official repository of \cite{schwarzer2023bigger}. The scripts in \pycode{rliable} truncate only the first ten runs from all independent runs for each game. The statistics thus may differ from those presented in Table \ref{tab:Atari_100K} for BBF-RR2 and BBF-RR8.}
  \label{fig:stats}
\end{figure*}

\subsection{The importance of variance reduction}
We experiment with a \pycode{seed} value of 3 in Algorithm \ref{alg:code}. The results are listed in Table \ref{tab:variance_reduction}. It is worth emphasizing that SAC-BBF without variance reduction differs from SAC-BBF in just \emph{one} line of code in the implementation.

Table \ref{tab:variance_reduction} demonstrates the impact of this single line of code. Without variance reduction, SAC-BBF shows a negative IQM in the tested environments, indicating that the agents might perform worse than a uniform random strategy. Previous attempts at generalizing SAC to discrete domains all missed this simple technique.

\subsection{Annealing $\beta$ and using a sampling strategy during evaluation}
We experiment with a \pycode{seed} value of 2 in Algorithm \ref{alg:code}. The results are listed in Table \ref{tab:anneal_random}.

\noindent\textbf{Annealing $\beta$:} The results in Table \ref{tab:anneal_random} indicate that annealing $\beta$ yields better results than using a constant value, as evidenced by the aggregate metrics such as IQM. In Seaquest and Alien, these variants produce comparable results.

\noindent\textbf{The impact of using a sampling strategy for evaluation:} 
As indicated in Table \ref{tab:anneal_random}, using sampling strategies generally leads to better results in terms of aggregate metrics, except for the ``Median'' metric. Regarding average scores, these two strategies are comparable in Seaquest, Alien, and Pong. Sampling strategies outperform greedy ones in CrazyClimber but underperform in Kangaroo.
After observing the raw scores, we find that the higher average score of greedy strategies in Kangaroo is due to a single score of 10400. The IQM metric is robust against outlier scores by discarding the bottom and top 25\%.
While these comparisons do not provide a definitive conclusion, they do indicate that using sampling strategies does not negatively impact performance.

We did not conduct experiments on using sampling or greedy strategies during the evaluation for SAC-BBF with annealing $\beta$. 
We anticipate using sampling strategies or not for this case would not significantly affect performance, as SAC-BBF runs with $\beta=0$ for the last $F=\mathrm{40K}$ training updates. So, the final strategy would not have very high entropy.

\subsection{Scores and aggregate metrics for SAC-BBF across the 26 Atari 100K games}
We present the results of SAC-BBF with RR 2 for the complete suite of Atari 100K benchmarks in Table \ref{tab:Atari_100K}. To provide a more straightforward visualization, in Fig.~\ref{fig:stats}, we utilize the open-source library \pycode{rliable} from \cite{agarwal2021deep} to illustrate these aggregate metrics with 95\% stratified bootstrap confidence intervals (CIs).

\noindent\textbf{BBF vs.~SAC-BBF with RR 2:} As shown in Table \ref{tab:Atari_100K}, SAC-BBF surpasses BBF across all aggregate metrics when restricted to only RR 2, which illustrates integrating SAC's generalization helps improve the learning of BBF agents.

\noindent\textbf{Comparison with BBF (RR 8):} By resetting the parameters fully or partially, RL agents exhibit RR scaling capabilities. In Table \ref{tab:Atari_100K}, we observe the improvement by comparing the IQMs of BBF with RR 8 and BBF with RR 2. Similar improvements exist in SR-SPR \cite{d2022sample}, where the IQM of SR-SPR starts at 0.444 for RR 2, improves to 0.589 for RR 8, and even reaches 0.632 for RR 16. Nonetheless, even with RR 8, which indicates four times the training time compared to RR 2, BBF exhibits lower IQM compared to SAC-BBF with RR 2.
This further underscores the efficacy of the SAC modules within the SAC-BBF framework.

We also observe from Table \ref{tab:Atari_100K} that BBF with RR 8 achieves the highest average scores in 10 environments, while this number is only 5 for SAC-BBF. Additionally, when considering the ``Optimality Gap'' metric, BBF with RR 8 outperforms SAC-BBF. However, when comparing the ``Median'' metric, BBF with RR 8 shows variability, as indicated by the ``Median'' statistic within one parenthesis in Table \ref{tab:Atari_100K} or as shown in Fig.~\ref{fig:stats}. These findings indicate that agents trained using BBF can still be highly competitive with increased RR. Nevertheless, we next demonstrate that increasing RR also further improves SAC-BBF.

\subsection{Comparison results on inference and training times}
We run the following experiment on a single RTX 4090 GPU with 24GB of memory.
The implementation of SAC-BBF follows the JAX implementation of BBF. Besides, we modify JAX's default GPU memory allocation strategy by \pycode{os.environ["XLA_PYTHON_CLIENT_MEM_FRACTION"] = "1."}.

\noindent\textbf{Inference time:} Despite incorporating additional modules, SAC-BBF demonstrates shorter inference time, as shown in Table \ref{tab:time}. During inference, SAC-BBF agents rely solely on the policy network. The ``policy head'' module depicted in Fig.~1 consists of a linear layer with an output dimension of $|\mathcal{A}|$. In contrast, BBF employs distributional RL, resulting in a ``Q-function head'' with an output dimension of $N \times |\mathcal{A}|$, where $N$ represents the number of atoms.

\noindent\textbf{Training time:} For RR 2, the introduction of additional modules in SAC-BBF only slightly increases the training time for training one agent in the ChopperCommand environment (the default environment in the official repository of BBF). The time difference is approximately 15 minutes. With an increased RR from 2 to 4, the number of training updates also doubles, resulting in a time difference of 30 minutes for RR 4. We note SAC-BBF-RR2 requires less than one-third of the training time compared to BBF-RR8.

\begin{table}
\caption{\textbf{Comparison of inference and training times between BBF and SAC-BBF.} The inference time is measured by repeatedly running the feed-forward process 3200 times with a batch size of 1.}
\label{tab:time}
\resizebox{\columnwidth}{!}{%
\begin{tabular}{r|rrrrrr}
\hline
\hline
 &
\begin{tabular}{r}
     BBF \\
     RR2
\end{tabular}
&
\begin{tabular}{r}
     SAC-BBF \\
     RR2 
\end{tabular}
&
\begin{tabular}{r}
     BBF \\
     RR4 
\end{tabular}
&
\begin{tabular}{r}
     SAC-BBF \\
     RR4 
\end{tabular} &
\begin{tabular}{r}
     BBF \\
     RR8 
\end{tabular} &
\begin{tabular}{r}
     SAC-BBF \\
     RR8 
\end{tabular} \\
\hline
Inference & 2.034 sec & 1.938 sec & - & - & - & - \\
Training & 92 min & 106 min & 177 min & 207 min & 371 min & 412 min \\
\hline
\end{tabular}
}
\end{table}

\section{Conclusion}
In this paper, we have explored the application of SAC in the context of discrete action spaces. By providing rigorous theoretical proofs, we present a discrete variant of SAC that works in standard maximum reward RL. It enables the integration of SAC with the state-of-the-art sample-efficient model-free algorithm BBF.
The resulting SAC-BBF is the \emph{only} model-free sample-efficient RL algorithm that introduces explicit policy heads into the Rainbow backbone.
Experimental results demonstrate the promising performance of the integration. With RR 2, the algorithm SAC-BBF achieves the highest IQM of 1.088. Additionally, SAC-BBF exhibits replay-ratio scaling capabilities, indicating the possibility of even better results by increasing replay ratios as in BBF.
We believe that SAC-BBF contributes to advancing the research on considering separate policy networks for model-free sample-efficient RL.



{
    \small
    \bibliographystyle{ieeenat_fullname}
    \bibliography{main}
}
\clearpage
\setcounter{page}{1}
\counterwithin{table}{section}
\numberwithin{equation}{section}
\maketitlesupplementary
\appendix

\section{Action selection}

\subsection{Action selection in BBF}
\noindent \textbf{Utilizing target networks for action selection:} 
BBF utilizes Rainbow as the underlying RL approach.
However, a distinction arises when generating transitions stored in the replay buffer: Rainbow relies on the \emph{online} parameters for action selection, whereas BBF employs the \emph{target} parameters for this purpose.
The BBF paper highlights the use of target networks, both for setting the target values during training (see Sec.~5.1) and for action selection, emphasizing its ``surprising importance.'' 

\noindent \textbf{Action selection in training:} 
In Rainbow, without Noisy Nets, it uses $\epsilon$-greedy strategies for exploration, where $\epsilon$ typically starts at 1, representing uniform random action selection, and gradually decreases to a much smaller final value, such as 0.01. In Rainbow, $\epsilon$ remains strictly greater than 0 to ensure exploration.

In contrast, BBF agents exhibit more aggressive behavior. In BBF, $\epsilon$ starts at 1. Then, it swiftly diminishes to 0 within just 4K steps of interactions. 4K steps constitute only a minor portion of the total allowance of 100K steps. So the question is: How do the agents in BBF maintain exploration after 4K steps, given that $\epsilon$ is 0, indicating purely greedy behavior?

We speculate that BBF maintains a certain level of exploration through two mechanisms: 1.~Action selection relies on target networks, which effectively act as an ensemble of previous online agents. 2.~Periodic resets of target parameters inject additional noises for action selection.

\noindent \textbf{Action selection in evaluation:} During evaluation, BBF follows the approach of Rainbow by setting $\epsilon$ to 0.001 for action selection, indicating that it does not employ a purely greedy strategy.

\subsection{Action selection in SAC-BBF}
\noindent \textbf{Action selection in training:} In SAC-BBF, partly to ensure a fair comparison, we utilize the same approach of employing target networks for action selection. Specifically, during interactions with the environment, SAC-BBF samples from $\pi_{\theta_\mathrm{targ}}(\cdot|\vs)$ for the received state $\vs$.


\noindent \textbf{Action selection in evaluation:} 
During evaluation, instead of using $\argmax_{\va'}\pi_{\theta_\mathrm{targ}}(\va'|\vs)$, SAC-BBF continues to sample from $\pi_{\theta_\mathrm{targ}}(\cdot|\vs)$. We note BBF is not purely greedy during evaluation either.
Experimental results across multiple environments indicate 
a sampling strategy during evaluation gives slightly superior performance.

\noindent \textbf{A note on the strength $\beta$ of entropy regularization:} In Sec.~4.3, we have already discussed the annealing process of $\beta$. By decreasing the value of $\beta$, we encourage the development of a more deterministic policy, akin to setting $\epsilon$ to zero in the $\epsilon$-greedy strategy of BBF. As SAC-BBF maintains a sampling strategy for evaluation, we prefer a more deterministic policy as the training phase nears its conclusion. To achieve this, we train SAC-BBF agents with $\beta=0$ for the final $F$ training steps, where $F$ is a hyperparameter kept constant for different RRs.


\section{Additional experimental results}

\begin{table}
\caption{\textbf{Comparison of improved results by increasing RR.} The human-normalized IQM and other statistics are over the five randomly selected environments. To compute the aggregate statistics, we use the publicly available scores for BBF. Therefore, the average scores for BBF-RR8 may vary from those presented in Table 2.}
\label{tab:sac-bbf-rr4}
\resizebox{\columnwidth}{!}{%
\begin{tabular}{rrrrr}
\hline
\hline
Game &
\begin{tabular}{r}
     BBF \\
     RR2
\end{tabular}
&
\begin{tabular}{r}
     BBF \\
     RR8 
\end{tabular}
&
\begin{tabular}{r}
     SAC-BBF \\
     RR2 
\end{tabular}
&
\begin{tabular}{r}
     SAC-BBF \\
     RR4 
\end{tabular} \\
\hline
KungFuMaster & 16616.857 & 17697.4 & 17746.9 & \textbf{20456.9} \\
Gopher & 1209.414 & \textbf{1407.387} & 1203.6 & 1320.08 \\
Krull & 8069.842 & 8383.532 & 7884.82 & \textbf{8495.41} \\
Asterix & 3169.785 & 4106.56 & 5641.45 & \textbf{7558.45} \\
Qbert & 3245.339 & 4318.82 & 4234 & \textbf{4777.5} \\
\hline
IQM ($\uparrow$) & 0.498 & 0.584 & 0.600 & \textbf{0.747} \\
Optimality Gap ($\downarrow$) & 0.453 & 0.397 & 0.390 & \textbf{0.316} \\
Median ($\uparrow$) & 0.441 & 0.533 & 0.654 & \textbf{0.886} \\
Mean ($\uparrow$) & 1.564 & 1.689 & 1.613 & \textbf{1.817} \\
\hline
\end{tabular}
}
\end{table}

\subsection{Improved results by increasing RR for SAC-BBF}
We conduct experiments using a \pycode{seed} value 1 for randomly selecting the testing environments. The results are listed in Table \ref{tab:sac-bbf-rr4}. The findings in Table \ref{tab:sac-bbf-rr4} validate the RR scaling capabilities of SAC-BBF agents. SAC-BBF-RR4 achieves the best results on all aggregate metrics for the five randomly chosen environments. On the other hand, when the RR is 2, SAC-BBF outperforms BBF-RR8 in terms of ``IQM,'' ``Optimality Gap,'' and ``Median,'' but falls behind in ``Mean.'' Nevertheless, SAC-BBF-RR2 still outperforms BBF-RR2 across all metrics.
Lastly, the training time of SAC-BBF-RR4 is still shorter than that of BBF-RR8, as indicated in the subsection followed.

\section{Proofs}
We rigorously follow the original proofs in SAC, so one can easily verify the correctness.

\subsection{Proof of Lemma \ref{lem:improve}}
\label{sec:proof_lemma_improve}

\begin{proof}
    Let $\pi_\mathrm{old}\in \Pi$ and let $Q^{\pi_\mathrm{old}}$ and $V^{\pi_\mathrm{old}}$ be the corresponding state-action value and state value. Since we can always choose $\pi_\mathrm{new}=\pi_\mathrm{old}\in\Pi$, the following holds:
    \begin{equation}
    \label{eq:x-inequality}
        \mathop{\E}_{\va_t\sim \pi_\mathrm{new}}[Q^{\pi_\mathrm{old}}(\vs_t,\va_t)]\ge \mathop{\E}_{\va_t\sim \pi_\mathrm{old}}[Q^{\pi_\mathrm{old}}(\vs_t,\va_t)]=V^{\pi_\mathrm{old}}(s_t)\text{.}
    \end{equation}

    Next, consider the Bellman equation:
    \begin{equation}
        \begin{split}
            &Q^{\pi_\mathrm{old}}(\vs_t,\va_t)\\
            =&r(\vs_t,\va_t)+\gamma \mathop{\E}_{\vs_{t+1}\sim p}[V^{\pi_\mathrm{old}}(\vs_{t+1})]\\
            \le & r(\vs_t,\va_t)+\gamma \mathop{\E}_{\vs_{t+1}\sim p}\left[\mathop{\E}_{\va_{t+1}\sim \pi_\mathrm{new}}Q^{\pi_\mathrm{old}}(\vs_{t+1},\va_{t+1})\right]\\
            \vdots &\\
            \le &Q^{\pi_\mathrm{new}}(\vs_t,\va_t)\text{,}
        \end{split}
    \end{equation}
    where we have repeatedly expanded $Q^{\pi_\mathrm{old}}$ on the RHS by applying the Bellman equation and the bound in Eq.~\ref{eq:x-inequality}. Convergence to $Q^{\pi_\mathrm{new}}$ follows from Lemma 1.
\end{proof}

\subsection{Proof of Theorem \ref{thm:policy_iteration}}
\label{sec:policy_iteration}

\begin{proof}
    Let $\pi_i$ be the policy in iteration $i$. By Lemma 2, the sequence $Q^{\pi_i}$ is increasing monotonically. As the reward is bounded, the value of $Q^\pi$ is bounded. Therefore, the sequence converges to some $\pi^*$. We will still need to show that $\pi^*$  is indeed optimal. At convergence, it must be the case that $\E_{\va_t\sim \pi^*}[Q^{\pi^*}(\vs_t,\va_t)]> \E_{\va_t\sim \pi}[Q^{\pi^*}(\vs_t,\va_t)]$ for all $\pi\in\Pi,\pi\neq \pi^*$. Using the same iterative argument as in the proof of Lemma 2, we get $Q^{\pi^*}(\vs_t,\va_t)>Q^\pi(\vs_t,\va_t)$ for all $(\vs_t,\va_t)\in\mathcal{S}\times\mathcal{A}$, that is, the value of any other policy in $\Pi$ is lower than that of the converged policy. Hence $\pi^*$ is optimal in $\Pi$.
\end{proof}

\subsection{Proof of Lemma \ref{lem:objective_grad}}
\label{sec:objective_grad}
\begin{proof}
\begin{equation}
\begin{split}
&\nabla_\theta \mathop{\E}_{\va_t\sim \pi_\theta}[Q_{\phi_\mathrm{old}}(\vs_t,\va_t)]\\
=&\nabla_\theta\sum_{\va_t\in\mathcal{A}}\pi_\theta(\va_t|\vs_t)Q_{\phi_\mathrm{old}}(\vs_t,\va_t)\\
=&\sum_{\va_t\in\mathcal{A}}\left[Q_{\phi_\mathrm{old}}(\vs_t,\va_t)\nabla_\theta\pi_\theta(\va_t|\vs_t)\right]\\
=&\sum_{\va_t\in\mathcal{A}}\left[\pi_\theta(\va_t|\vs_t)Q_{\phi_\mathrm{old}}(\vs_t,\va_t)\nabla_\theta\log\pi_\theta(\va_t|\vs_t)\right]\\
=&\mathop{\E}_{\va_t\sim \pi_\theta}\left[Q_{\phi_\mathrm{old}}(\vs_t,\va_t)\nabla_\theta\log\pi_\theta(\va_t|\vs_t)\right]\text{,}
\end{split}
\end{equation}
where we only show the derivation for the case of discrete action spaces. A similar derivation follows by replacing $\sum$ with $\int$, given that no reparameterization trick is used as in SAC.
\end{proof}

\subsection{Proof of Lemma \ref{lem:variance_reduction}}
\label{sec:variance_reduction}
\begin{proof}
\begin{equation}
\begin{split}
    &\mathop{\E}_{\va_t\sim \pi_\theta}\biggl[\Bigl(Q_{\phi_\mathrm{old}}(\vs_t,\va_t)-\sum_{\va'\in\mathcal{A}}\pi_{\theta_\mathrm{old}}(\va'|\vs_t)Q_{\phi_\mathrm{old}}(\vs_t,\va')\Bigr)\\   &\quad\quad\quad\quad\nabla_\theta\log\pi_\theta(\va_t|\vs_t)\biggr]\\
    =&\mathop{\E}_{\va_t\sim \pi_\theta}[Q_{\phi_\mathrm{old}}(\vs_t,\va_t)\nabla_\theta\log\pi_\theta(\va_t|\vs_t)]-\\
    &\Bigl(\sum_{\va'\in\mathcal{A}}\pi_{\theta_\mathrm{old}}(\va'|\vs_t)Q_{\phi_\mathrm{old}}(\vs_t,\va')\Bigr)\mathop{\E}_{\va_t\sim \pi_\theta}[\nabla_\theta\log\pi_\theta(\va_t|\vs_t)]\\
    =&\mathop{\E}_{\va_t\sim \pi_\theta}[Q_{\phi_\mathrm{old}}(\vs_t,\va_t)\nabla_\theta\log\pi_\theta(\va_t|\vs_t)]-\\
    &\quad\quad\quad\quad\Bigl(\sum_{\va'\in\mathcal{A}}\pi_{\theta_\mathrm{old}}(\va'|\vs_t)Q_{\phi_\mathrm{old}}(\vs_t,\va')\Bigr)\nabla_\theta 1\\
    =&\mathop{\E}_{\va_t\sim \pi_\theta}[Q_{\phi_\mathrm{old}}(\vs_t,\va_t)\nabla_\theta\log\pi_\theta(\va_t|\vs_t)]\text{.}
\end{split}
\end{equation}
We use the topmost term in the above equation as the gradient estimation used in the code implementation.
\end{proof}


\end{document}